\def\year{2017}\relax
\documentclass[letterpaper]{article}
\usepackage{aaai17}
\usepackage{times}
\usepackage{helvet}
\usepackage{courier}
\usepackage[utf8]{inputenc} 
\usepackage[T1]{fontenc}    
\usepackage{url}            
\usepackage{booktabs}       
\usepackage{amsfonts}       
\usepackage{nicefrac}       
\usepackage{microtype}      
\usepackage[longnamesfirst]{natbib}
\usepackage{url}
\usepackage{CJKutf8}
\usepackage{xcolor}
\usepackage{subfigure}
\usepackage{listings}
\usepackage[cmex10]{amsmath}
\usepackage{amsthm}
\usepackage{graphicx}
\usepackage{multirow}
\usepackage{bbding}
\usepackage{amsfonts,amssymb}
\usepackage{bm}
\usepackage{algorithm}
\usepackage{algorithmic}
\usepackage{epstopdf}
\usepackage{enumerate}
\newtheorem{theorem}{Theorem}
\newtheorem{lemma}{Lemma}

\newtheorem{proposition}{Proposition}
\newtheorem{corollary}{Corollary}
\newtheorem{definition}{Definition}

\usepackage{times}
\begin{document}
\title{A Unified Convex Surrogate for the Schatten-$p$ Norm}
\author{Chen Xu, Zhouchen Lin\thanks{Corresponding author.}, Hongbin Zha\\
	Key Laboratory of Machine Perception (MOE), School of EECS, Peking University, P. R. China\\
  Cooperative Medianet Innovation Center, Shanghai Jiao Tong University, P. R. China\\
	{\tt\small xuen@pku.edu.cn, zlin@pku.edu.cn,  zha@cis.pku.edu.cn}
}
\maketitle
	\begin{abstract}
		The Schatten-$p$ norm ($ 0< p < 1$) has been widely used to replace the nuclear norm for better approximating the rank function. However, existing methods are either 1) not scalable for large scale problems due to relying on singular value decomposition (SVD) in every iteration, or 2) specific to some $p$ values, e.g., $1/2$, and $2/3$. In this paper, we show that for any $p, p_1$, and $p_2 >0$ satisfying $1/p=1/p_1+1/p_2$, there is an equivalence between the Schatten-$p$ norm of one matrix and the Schatten-$p_1$ and the Schatten-$p_2$ norms of its two factor matrices. We further extend the equivalence to multiple factor matrices and show that all the factor norms can be \textit{convex and smooth} for any $p>0$. In contrast, the original Schatten-$p$ norm for $0<p<1$ is non-convex and non-smooth. As an example we conduct experiments on matrix completion. To utilize the convexity of the factor matrix norms, we adopt the accelerated proximal alternating linearized minimization algorithm and establish its sequence convergence. Experiments on both synthetic and real datasets exhibit its superior performance over the state-of-the-art methods. Its speed is also highly competitive.
	\end{abstract}
	\section{Introduction}
 \noindent In recent years, low rank matrix minimization has found wide applications, e.g., matrix completion \citep{mcproblem}, low rank representation \citep{lrr}, multi-task learning \citep{lifted}, etc. Often, we can formulate the problem as follows: 
	\begin{equation} \label{original}
		\min_X F(X) = \min_{X} f(X)+ \lambda \Omega(X),
	\end{equation}
	where $f(\cdot): \mathbb{R}^{m \times n} \rightarrow \mathbb{R}^{+}$  is the loss function, $\Omega( \cdot ): \mathbb{R}^{m \times n}  \rightarrow \mathbb{R}^{+}$ is the spectral regularization \citep{newspectra} which ensures low rankness, and $\lambda \in \mathbb{R}^+$  balances the two terms. 
	
	As the tightest convex envelop of rank function on the unit ball of matrix operator norm, the nuclear norm regularizer is often suggested for $\Omega(X)$ \citep{guaranteed,candesandtao}. In fact, the nuclear norm is the $\ell_1$-norm on the vector of singular values. It achieves low rankness by encouraging sparseness on the singular values. As \citet{fanli} pointed out, the $\ell_1$-norm is a loose approximation to the $\ell_0$-norm and overpenalizes large entries of vectors. By an analogy between the rank function of matrices and the $\ell_0$-norm of vectors, the nuclear norm also overpenalizes large singular values.  As a tighter approximation to the rank function, the Schatten-$p$ quasi-norm $(0 <p <1)$ is suggested to replace the nuclear norm \citep{nie}. For the task of matrix completion, the Schatten-$p$ quasi-norm has empirically shown to be  superior to the nuclear norm. Moreover, \citet{pisometry} theoretically prove that for the matrix completion problem the Schatten-$p$ quasi-norm with a small $p$ requires much fewer observed entries than the nuclear norm minimization does.    
	
	However,  the Schatten-$p$ quasi-norm is non-convex and non-smooth. So the optimization for problem \eqref{original} is  much more challenging. Recently,  \citet{iruclq} propose iterative reweighted least square (IRucLp) to solve  a smoothed subproblem by approximating the Schatten-$p$  quasi-norm at each iteration. They prove that any limit point of the generated sequence is a stationary point. Moreover, \citet{irnn} propose the iterative reweighted nuclear norm (IRNN) algorithm. Besides the Schatten-$p$ quasi-norm, IRNN is able to tackle a variety of regularizations, e.g., MCP \citep{mcp} and SCAD \citep{fanli}, on the singular values. A  similar convergence result as IRcuLq is also established. However, both of the algorithms involve computing SVD at each iteration, which is expensive for large-scale problems. 
	
	Alternative to \eqref{original}, the bilinear factorization with two factor matrix norm regularizers  is suggested \citep{maximummagine,unifying,scalableAlgorithm}:
	\begin{equation}\label{factor}
		\min_{U,V} F(U,V) = \min_{U,V} f(UV^T)+ \lambda \left(\Omega_u(U)+\Omega_v(V)\right),
	\end{equation}
	where $U \in \mathbb{R}^{m \times d}$ and  $V \in \mathbb{R}^{n \times d}$ are the unknown factor matrices. Quite often, $d \ll \min \{m,n\}$ holds. When minimizing \eqref{factor}, one only needs to operate on two much smaller factor matrices in contrast to the full dimensional $X$ as \eqref{original}. Thus \eqref{factor} is better suited for large-size applications. 
	As \citet{maximummagine} indicated, when $\Omega_u(U)+ \Omega_v(V) =\|U\|_F^2/2 + \|V\|_F^2/2$, it can  be equivalently represented as the surrogate of  $\Omega(X)=\|X\|_*$ when enforcing $X=UV^T$.  Let $r^*$ denote  the rank of the optimal $X^*$ in \eqref{original},  \citet{spectralreg} proved that  the minimum objective function values of \eqref{original} and \eqref{factor} are equal once  $d \geq r^*$.  Quite recently, \citet{scalableAlgorithm, aistats} extended the surrogate of the nuclear norm regularizer $\Omega(X)$ to that of specific Schatten-$p$ norms, where $p=1/3$, $1/2$, or $2/3$. They proposed to use the proximal alternating linearized minimization (PALM) algorithm and established its sequence convergence. Motivated by these results, we  further extend the surrogate to the general Schatten-$p$ norm. The contributions of this paper are as follows:
	\begin{enumerate}[(a)]
		\item We show that for any $p, p_1$, and $p_2>0$ satisfying $1/p=1/p_1+1/p_2$, there is an equivalence between the Schatten-$p$ norm of $X$ and the Schatten-$p_1$ and the Schatten-$p_2$ norms of $U$ and $V$ when enforcing $X=UV^T$ (See Theorem \ref{theorem1}). The existing surrogates for $p=1$, $1/2$, and $2/3$ are only special cases of ours. We also give an \textit{entirely different and much simpler} proof than the existing ones.
		\item We extend the above result to multi-factor matrices (See Corollary \ref{corollary2}) and show that  each factor matrix norm of the surrogate can be \textit{convex and smooth} for any $p>0$. In contrast, the Schatten-$p$ norm $(0<p<1)$ is non-smooth and non-convex, and the results of \citet{scalableAlgorithm, aistats} are only limited to two or three-factor cases which all involve the non-smooth nuclear norm.  
		\item We unify the minimization of \eqref{original} and \eqref{factor} for  general Schatten-$p$ norm regularizers, where the former is reformulated to  the latter (See Theorem \ref{equivalence}).  We also show that the factorization formulation should be preferred when $0<p<1$.
		\item  We conduct experiments on matrix completion as an example to test our framework.  By incorporating the convexity of the factor matrix norms, our accelerated proximal alternating algorithm achieves state-of-the-art performance. We also prove its sequence convergence.   
	\end{enumerate}
	\section{Notations and Background}\label{notation}
	Consider the SVD of a matrix $X \in \mathbb{R}^{m \times n}$: $X=U_X \mathrm{diag}\left(\sigma_i(X)\right) V_X^T$, where $\sigma_i(X)$ denotes its $i$-th singular value in descending order. Then the Schatten-$p$ norm $(0< p < \infty)$ of  $X$ is defined as
	\begin{equation}\label{schatten}
		\|X\|_{S_p} \triangleq  \left( \sum_{i=1}^{\min\{m,n\}} \sigma_i^p(X)  \right)^{\frac{1}{p}}.
	\end{equation}
	Special cases of the Schatten-$p$ norm  include  
	the nuclear norm ($p=1$)  and  the Frobenius norm ($p=2$). When $p \geq 1$, $\|X\|_{S_p}^p$ is convex w.r.t. $X$. When $p > 1$, $\|X\|_{S_p}^p$ is further differentiable everywhere with the gradient being  $\nabla_X \|X\|_{S_p}^p= p U_X \mathrm{diag}\left(\sigma_i^{p-1}(X)\right)V_X^T$ \citep{gradientofnorm}. 
	
	The proximal mapping of $\|X\|_{S_p}^p$ is defined as:
	\begin{equation}\label{proximadefine}
		\mathbf{Prox}_{\lambda, p}(Y)= \arg \min_X \frac{1}{2} \|X-Y\|_F^2+ \frac{\lambda}{p} \|X\|_{S_p}^p
	\end{equation} 
	\begin{lemma}\citep{gsvt} \label{proximalsp}
		Let $Y=U_Y diag\left(\sigma_i(Y)\right) V_Y^T$ be the SVD of $Y \in \mathbb{R}^{m \times n}$ with $\{\sigma_i(Y)\}$ in descending order. Then we have  
		\begin{equation}
			\mathbf{Prox}_{\lambda, p}(Y)= U_Y \mathrm{diag}(\hat \sigma_i)V_Y^T, 
		\end{equation}
    where $\hat \sigma_i$ is defined as the scalar proximal mapping in \eqref{proximadefine}:
    \begin{equation}\label{sigmaproximal}
     \hat \sigma_i=\mathbf{Prox}_{\lambda, p}(\sigma_i(Y)).
    \end{equation}
	\end{lemma}
	When  $p \geq 1$, problem \eqref{sigmaproximal} is strongly convex. Thus it can be easily solved  by off-the-shelf algorithms. There are some special cases of $p$ that have closed-form solutions, e.g. $\hat \sigma_i = \max(\sigma_i(Y)- \lambda, 0)$ if $p=1$, which is known as the soft-thresholding \citep{donoho},  and  $\hat \sigma_i = \sigma_i(Y)/(1+\lambda)$ when $p=2$ by making the derivative of the objective function zero.  When $p <1$,  problem \eqref{sigmaproximal} becomes non-smooth and non-convex. By considering its structure  properly, Zuo et al. \citep{gisa} proposed the generalized iterated shrinkage algorithm (GISA), which solved  \eqref{sigmaproximal} efficiently with high precision.  Thus in this paper, we regard the proximal mapping of \eqref{proximadefine} for any $0<p <\infty$ as an easy problem. 
	
	For the analysis on convergence of algorithms for problems with  non-convex objectives, we need the definition of critical points given in \citep{palm}:
	\begin{definition} \textbf{(Critical Points)}
		Let a non-convex function $g: \mathbb{R}^{n} \rightarrow (\infty, +\infty]$ be a proper and lower semi-continuous function, and $\mathrm{dom}~g =\{x \in \mathbb{R}^n: g(x)< + \infty\}$.
		\begin{itemize}
			\item For any $x \in \mathrm{dom}~g$, the Fr\'{e}chet sub-differential of $g$ at $x$ is defined as 
			\begin{equation}
			\begin{split}
				\hat{\partial}g(x)&=\{u\in \mathbb{R}^n:\\&\lim_{y\neq x}\inf_{y\rightarrow x} \frac{g(y)-g(x)-\left<u,y-x\right>}{\|y-x\|_{2}} \geq 0\},
			\end{split}
			\end{equation}
			and 
			\begin{equation}
				\hat{\partial}g(x)=\emptyset~\text{if}~x\not \in \mathrm{dom}~g.
			\end{equation}
			\item The points whose sub-differential contains $0$ are called critical points, i.e., a  point $x$ is a critical point of $g$ if $0\in \partial g(x).$
		\end{itemize}
		
	\end{definition}
	\section{Unified Surrogate for Schatten-$p$ Norm} 
	Before giving our unified surrogate for the Schatten-$p$ norm, we review three existing surrogates for  specific $p$ values, i.e., $p=1, 2/3$, and  $1/2$.
	\begin{proposition}\label{bifro} (\textbf{Bi-Frobenius Norm Surrogate} \citep{maximummagine,spectralreg})
		Given matrices $U \in \mathbb{R}^{m \times d}$, $V \in \mathbb{R}^{n \times d}$, and $X \in \mathbb{R}^{m \times n}$ with $rank(X)=r \leq d$, the following holds:
		\begin{equation}
			\begin{split}
				\|X\|_{*}=& \min_{U,V: X=UV^T} \frac{\|U\|_{F}^2}{2} +\frac{\|V\|_{F}^2}{2}.
			\end{split} 
		\end{equation}    
	\end{proposition}
	
	\begin{proposition}\label{binuc} (\textbf{Frobenius/Nuclear and Bi-Nuclear Norm Surrogate} \citep{scalableAlgorithm, aistats})
		Given matrices $U \in \mathbb{R}^{m \times d}$, $V \in \mathbb{R}^{n \times d}$, and $X \in \mathbb{R}^{m \times n}$ with $rank(X)=r \leq d$, the following holds:
		\begin{equation}
			\begin{split}
				\frac{3}{2}\|X\|_{S_{2/3}}^{2/3}=& \min_{U,V:X=UV^T} \|U\|_{*} +{\frac{1}{2}}\|V\|_{F}^2 \quad  \text{and} \\
				2\|X\|_{S_{1/2}}^{1/2}=& \min_{U,V:X=UV^T} \|U\|_{*} +\|V\|_{*}.
			\end{split} 
		\end{equation}
	\end{proposition}
	Note that we have rewritten  Proposition \ref{binuc}  in a more consistent way than the original one in \citep{scalableAlgorithm, aistats}. Combining the above two propositions, one may induce that there are some unified surrogates for Schatten-$p$ norm. In fact, we have:

	
	\begin{theorem} \label{theorem1} (\textbf{Bi-Schatten-$p$ Norm Surrogate})
		Given matrices $U \in \mathbb{R}^{m \times d}$, $V \in \mathbb{R}^{n \times d}$, and $X \in \mathbb{R}^{m \times n}$ with $rank(X)=r \leq d$, for any $p,p_1$ and $p_2 > 0$ satisfying ${1}/{p}= {1}/{p_1}+{1}/{p_2}$, we have 
		\begin{equation} \label{general1}
			\begin{split}
				\frac{1}{p} \|X\|_{S_p}^p
				=& \min_{U,V: X=UV^T} \frac{1}{p_1}\|U\|_{S_{p_1}}^{p_1}+ \frac{1}{p_2} \|V\|_{S_{p_2}}^{p_2}.
			\end{split} 
		\end{equation}
	\end{theorem}
	We provide an \textit{entirely different and much simpler} proof than those in \citep{scalableAlgorithm, aistats, spectralreg}, which use the property of the specific $p$, $p_1$, and $p_2$ values shown in Propositions \ref{bifro} and \ref{binuc} \footnote{In fact, we have tried to extend the proof in \citep{scalableAlgorithm, aistats} to the general case and  found that it needs to ensure one of $p_1$ or $p_2$ to be greater than $1$, which is less general compared with Theorem \ref{theorem1}.}. The core idea is to utilize the property of general Schatten-$p$ norms derived from the determinant of matrices:
	\begin{lemma} \label{lemma1} \citep{topics}[Theorem 3.3.14 (c)]
		For any matrices $A\in \mathbb{R}^{m \times l}$ and $B \in \mathbb{R}^{n \times l}$, denoting $\{\sigma_i(\cdot)\}$ as the singular values in descending order, we have 
		\begin{equation}
			\begin{split}
				&\sum_{i=1}^{\min\{m,n,l\} } \sigma_i^p(AB^T) \leq  \sum_{i=1}^{\min\{m,n,l\}} \sigma_i^p(A) \sigma_i^p(B),~~ \forall ~p>0.  
			\end{split}
		\end{equation}
	\end{lemma}
	For the completed proof of Theorem \ref{theorem1}, please refer to the Supplementary Material\footnote{All proofs in this paper are in Supplementary Material.}. By extending Theorem \ref{theorem1}  to multiple factors, we have \footnote{When we were preparing the camera ready version of this paper, Shang et al. told us that they also got the same result independently \citep{multiformulation}. But their proof still followed that in \citep{scalableAlgorithm, aistats}.}: 
	\begin{corollary} \label{corollary2}(\textbf{Multi-Schatten-$p$ Norm Surrogate})
		Given $I~(I \geq 2) $ matrices $X_{i},~i=1,\ldots, I,$ where $X_1 \in  \mathbb{R}^{m \times d_1}$, $X_{i} \in \mathbb{R}^{d_{i-1} \times d_{i}},  i=2,\ldots, I-1, $ $X_I \in  \mathbb{R}^{d_I \times n}$, and $X \in \mathbb{R}^{m \times n}$ with $rank(X)=r \leq \min \{d_i, i=1,\ldots, I\}$, for any $p, p_1, \ldots,p_I > 0$ satisfying ${1}/{p}= \sum_{1}^I 1/p_i$, we have 		
		\begin{equation} \label{general2}
			\begin{split}
				\frac{1}{p} \|X\|_{S_p}^p
				= \min_{X_i:X=\prod_{i=1}^{I}X_i} \sum_{i=1}^{I} ~\frac{1}{p_i}\|X_i\|_{S_{p_i}}^{p_i}.
			\end{split} 
		\end{equation}
	\end{corollary}
	For the ease of  computation, we assume that $d_i = d,~i=1,\ldots, I$ in the rest of this paper. Special cases of Corollary \ref{corollary2} include $(p=1/I, p_i=1,~ i=1,\ldots, I)$, e.g., Tri-Nuclear norm surrogate in \citep{aistats}, and   $(p=2/I, p_i=2, ~ i=1,\ldots, I)$. When $I=2$, Corollary \ref{corollary2} reduces to Theorem \ref{theorem1}. Corollary \ref{corollary2}  can be proved by induction, using Theorem \ref{theorem1}. In fact, the two do not just differ in the number of factors.  Corollary \ref{corollary2} enables us to choose some particular $p_i$ values which endow the operator $\| \cdot\|_{S_{p_i}}^{p_i}$ with nice properties, especially when $0<p< 1$, the case that we are mainly interested in.
	
	\begin{proposition}
		For any $0 < p<1$, there always exist $I \in N$ and  $p_i$ such that $1/p =  \sum_i^{I} 1/ p_i$, where all $p_i$ satisfy one of the cases:
		(a)  $p_i\geq 1$ or (b) $p_i > 1$.
	\end{proposition}
	\begin{figure}[ht]
		\centering
		\includegraphics[width=0.5\linewidth, height=0.4\linewidth]{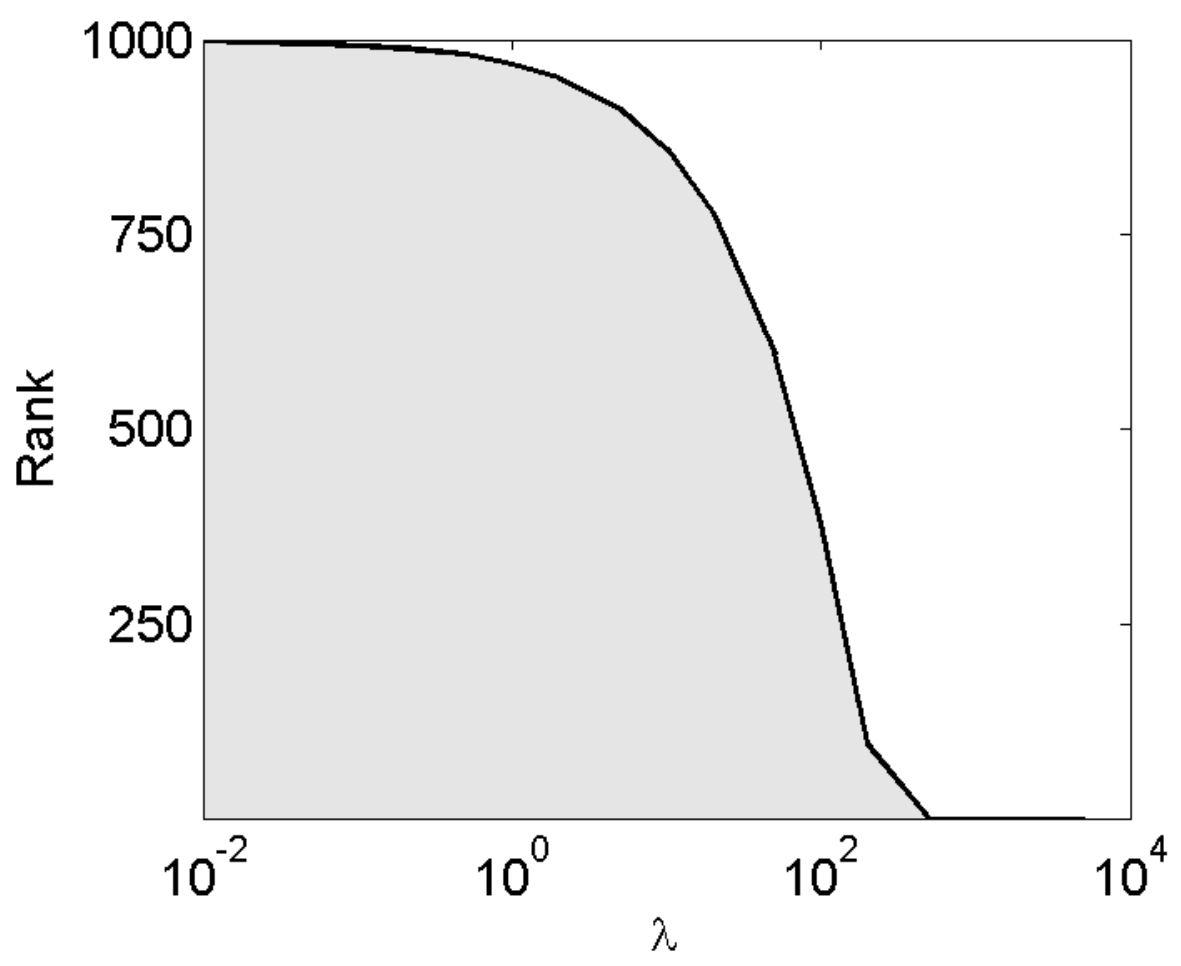}
		\caption{Region of equivalence between \eqref{schattenreg} and \eqref{schattenfac} on proximal mapping \eqref{proximadefine} with $Y$ being a $1000  \times 1000$ random matrix and $p=1/2$. When \eqref{schattenfac} is initialized in the white area, it is equivalent to the result obtained with \eqref{schattenreg} (black line).
			When the rank is known a priori, better reconstruction results can
			be found in the grey area, by using the factorization formulation \citep{unifying}.}
		\label{contained} 
	\end{figure}  
	When ensuring condition (a), the operator $\| \cdot\|_{S_{p_i}}^{p_i}$ becomes convex as mentioned before. In the later experiments, we will show that such convexity can be employed for acceleration.  When ensuring condition (b), the operator becomes differentiable. In other words, we transform the original \textit{non-smooth} function into a \textit{smooth} one. Thus it is possible to utilize some gradient-based methods, which give much freedom to the optimization. 
	
	By substituting the spectral regularization in \eqref{original} with the Schatten-$p$ norm, we have  
	\begin{equation}\label{schattenreg}
		\min_X F(X)=  \min_X f(X) + \frac{\lambda}{p} \|X\|_{S_p}^p. 
	\end{equation} 
	According to Corollary \ref{corollary2}, by rewriting $X$ in the multi-linear form as \eqref{general2}, the above problem becomes 
	\begin{equation}\label{schattenfac}
		\min_{\mathcal{X}} F(\mathcal{X})=  \min_{X_i,~i=1,\ldots, I} f\left(\prod_{i=1}^{I}X_i)\right) + \sum_{i=1}^{I} ~\frac{\lambda}{p_i}\|X_i\|_{S_{p_i}}^{p_i},
	\end{equation} 
	where $1/p=\sum_{1}^{I}1/p_i$ and $X_i$ is defined as Corollary \ref{corollary2}.  $\mathcal{{X}}= (X_1,X_2,\ldots,X_I)$  denotes the set of all unknown $X_i$'s. By mild modification on \citep{spectralreg}[Theorem 3] for the Bi-Frobenius norm surrogate, we have the following connections between \eqref{schattenreg} and \eqref{schattenfac}: 
	
	
	\begin{theorem}\label{equivalence} 
		Suppose $\hat{X}^*$ is a solution to \eqref{schattenreg}, and let $r^*$ be its rank. If $d \geq r^*$, the solutions of \ref{schattenreg} and \ref{schattenfac} are equivalent.  For any solution $\hat{\mathcal{X}}$ to \eqref{schattenfac}, $\prod_{i=1}^{I} \hat{X}_i$ is a solution to \eqref{schattenreg}. On the other hand, the SVD of $\hat{X}^*=\hat{U}_X^*\hat{\Sigma}_X^*\hat{V}_X^{*T}$ provides one such solution to \eqref{schattenfac} with $\hat{X}_1= \hat{U}_X^*\hat{\Sigma}_X^{* p/p_1}$, $\hat{X}_i=\hat{\Sigma}_X^{* p/p_i}, i=2,\ldots, I-1$, and $\hat{X}_I=\hat{\Sigma}_X^{* p/p_I}\hat{V}_X^{*T}$. 
	\end{theorem}

	The matrix factorization formulation \eqref{schattenfac} defines a bi-parameterized family of models indexed by $(d, \lambda)$, while the spectral penalty formulation \eqref{schattenreg} defines a uni-parameterized family. As Theorem \ref{equivalence} indicates, this family is a special path in the two-dimensional grid of solutions $\mathcal{\hat{X}}_{d,\lambda}$. Figure \ref{contained} shows the relationship. In real applications, it often occurs that the intrinsic rank $ r \ll \min\{m,n\}$.  When $r$ is unknown, we may overestimate $d \geq r$,  while $d  \ll \min\{m,n\}$ still holds. When $r$ is known a priori, e.g., Structure from Motion in computer vision, better reconstruction results can be found by setting $d=r$ in \eqref{schattenfac} than using \eqref{schattenreg} \citep{unifying}. In addition, the factorization formulation requires far less memory on the unknowns $\left( \mathcal{O}\left(d(m+n)\right) \ll \mathcal{O}(mn) \right)$. It also avoids the SVD computation on the full matrix, whose cost is as large as $\mathcal{O}(\min\{m,n\}mn)$.  For $0 < p <1$, \eqref{schattenreg} is a non-smooth and non-convex problem. In contrast, by choosing some appropriate $p_i$'s as discussed earlier, \eqref{schattenfac} can be smooth\footnote{If $f(\cdot)$ is also smooth.} (although still non-convex) and hold some good properties for acceleration. In conclusion, it is preferable to model the Schatten-$p$ norm based problem as \eqref{schattenfac} instead of \eqref{schattenreg}.    
	
	\section{Optimization on Matrix Completion}
	In this section, as a concrete example we consider solving the matrix completion problem.  Then problem \eqref{schattenfac} can be written as follows:
	\begin{equation} \label{matrixcompletion}
	\begin{split}
		\min_{\mathcal{X}} F(\mathcal{X}) = \min_{X_i,~i=1,\ldots, I} &\frac{1}{2} \left\|W \odot \left(M - \prod_{i=1}^{I}X_i\right) \right\|_F^2 \\& +\sum_{i=1}^{I} ~\frac{\lambda}{p_i}\left\|X_i \right\|_{S_{p_i}}^{p_i},
	\end{split}
	\end{equation}
	which is a non-convex problem, where $M\in \mathbb{R}^{m \times n}$ is the low rank measurement matrix.
	$W$ is a 0-1 binary mask with the same size as $M$. The entry
	value of $W$ being $0$ means that the component at the same
	position in $M$ is missing, and $1$ otherwise. The operator
	$\odot$ is the Hadamard element-wise product.  By utilizing the smoothness of the first part in \eqref{matrixcompletion}, we use the PALM proposed in \citep{palm}, which can  also be regarded as block coordinate descent (BCD) of Gauss-Seidel type. At each iteration, PALM minimizes $F$  cyclically over each of $X_1,\ldots, X_I$ while fixing the remaining blocks at their last updated values. Let $X^k_i$ denote the value of $X_i$ at the $k$-th update, $A_{-i}^k=X^{k}_1 \cdots X^{k}_{i-1}$ and $A_{+i}^{k-1}=X^{k-1}_{i+1}\cdots X^{k-1}_{I}$.  Then we can represent the first part of each subproblem minimizing $F$ as follows: 
	\begin{equation}
		f_i^{k}(X_i)=\frac{1}{2}\|W \odot(M-A_{-i}^k  X_i A_{+i}^{k-1})\|_F^2.
	\end{equation} 
	By further linearizing $f_i^k(X_i)$ at some point $\hat{X}^{k-1}_i$, the subproblem becomes    
	\begin{equation} \label{matrixsubproblems}
		\begin{split}
		\min_{X_i} & \left<\nabla f_i^k(\hat{X}_i^{k-1}), X_i-\hat{X}_i^{k-1}\right> \\ &+ \frac{L_i^{k-1}}{2}\|X_i-\hat{X}_i^{k-1}\|_F^2 + ~\frac{\lambda}{p_i}\|X_i\|_{S_{p_i}}^{p_i},
		\end{split}
	\end{equation}  
	which can be formulated as the proximal mapping \eqref{proximadefine} and solved efficiently or even in closed-form solution for specific $p_i$ values. $\nabla f_i^k (\hat{X}_i^{k-1})$ is the gradient of $f_i^k(X_i)$ at $\hat{X}_i^{k-1}$:
	\begin{equation} \label{matrixgradient}
	\begin{split}
		&\nabla f_i^k(\hat{X}_i^{k-1})= \\&(A_{-i}^k)^T \left(W \odot(M-A_{-i}^k \hat{X}_i^{k-1} A_{+i}^{k-1})\right)(A_{+i}^{k-1})^T.
	\end{split}
	\end{equation}        
	$L_i^{k-1}$ is the Lipschitz constant of $\nabla f_i^{k}(X_i)$:  
	\begin{equation}\label{matrixl}
		L_i^{k-1}=\max\{\|A_{-i}^{k}\|_2^2 \|A_{+i}^{k-1}\|_2^2 , \epsilon \},
	\end{equation} 
	where $\|A\|_2$ denotes the largest singular value of $A$ and $\epsilon >0$ is some small constant ensuring that $L^{k-1}_i$ is bounded away from $0$ for convergence. For the basic version of PALM, we let $\hat{X}_i^{k-1}= X_i^{k-1}$. When $p_i \geq 1$, the linearized subproblem \eqref{matrixsubproblems} becomes convex. The acceleration technique proposed by \citet{abcd} can then be adopted, where $\hat{X}_i^{k-1}$ is extrapolated as
	\begin{equation}\label{hatx}
		\hat{X}_i^{k-1}=X_i^{k-1}+ w_i^{k-1} (X_i^{k-1}-X_i^{k-2}),
	\end{equation} 
	where $w_i^{k-1}$ is defined as
	\begin{equation}\label{matrixw}
		w_i^{k-1}= \min \left\{\frac{t_{k-1}-1}{t_k}, 0.9999 \sqrt{\frac{L_i^{k-2}}{L_i^{k-1}}}\right\}
	\end{equation}
	with $t_0=1$ and $t_k=\left(1+\sqrt{1+4t_{k-1}^2}\right)/2$.
	\renewcommand{\algorithmicrequire}{\textbf{Input:}}
	\renewcommand{\algorithmicensure}{\textbf{Output:}}
	\begin{algorithm}[t]
		\caption{Minimizing $F(\mathcal{X})$ in \eqref{matrixcompletion} with accelerated PALM.}
		\label{algorithm}
		\begin{algorithmic}[1]
			\REQUIRE  $k=1$ and $X_i^{-1}=X_i^{0}, i= 1,\ldots, I.$
			\WHILE{ not converged}
			\FOR{ $i= 1,2, \ldots, I$}
			\STATE Compute $L_i^{k-1}$ as \eqref{matrixl} and  $w_i^{k-1}$ as \eqref{matrixw}.
			\STATE Update $\hat{X}_{i}^{k-1}$ as \eqref{hatx}.
			\STATE Update $X_i^k$ by solving \eqref{matrixsubproblems}.
			\ENDFOR
			\IF{$F(\mathcal{X}^k) \geq F(\mathcal{X}^{k-1}) $ }
			\STATE Reupdate $X_i^k$ by solving \eqref{matrixsubproblems}  with $\hat{X}_i^{k-1}= X_i^{k-1},~~ i=1, \ldots, I. $
			\ENDIF
			\STATE $k=k+1$
			\ENDWHILE
			\ENSURE  The factors $(X_1, \ldots, X_I)$
		\end{algorithmic}
	\end{algorithm}
	
		\begin{figure*}[ht!]
			\centering
			\includegraphics[width=0.5\linewidth]{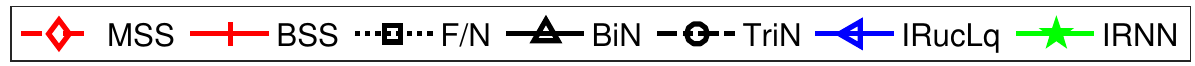} {\vspace{0.07in}}
			\begin{tabular}{cccc}
				{\hspace{-13pt}} \includegraphics[width=0.24\linewidth,height=0.21\linewidth]{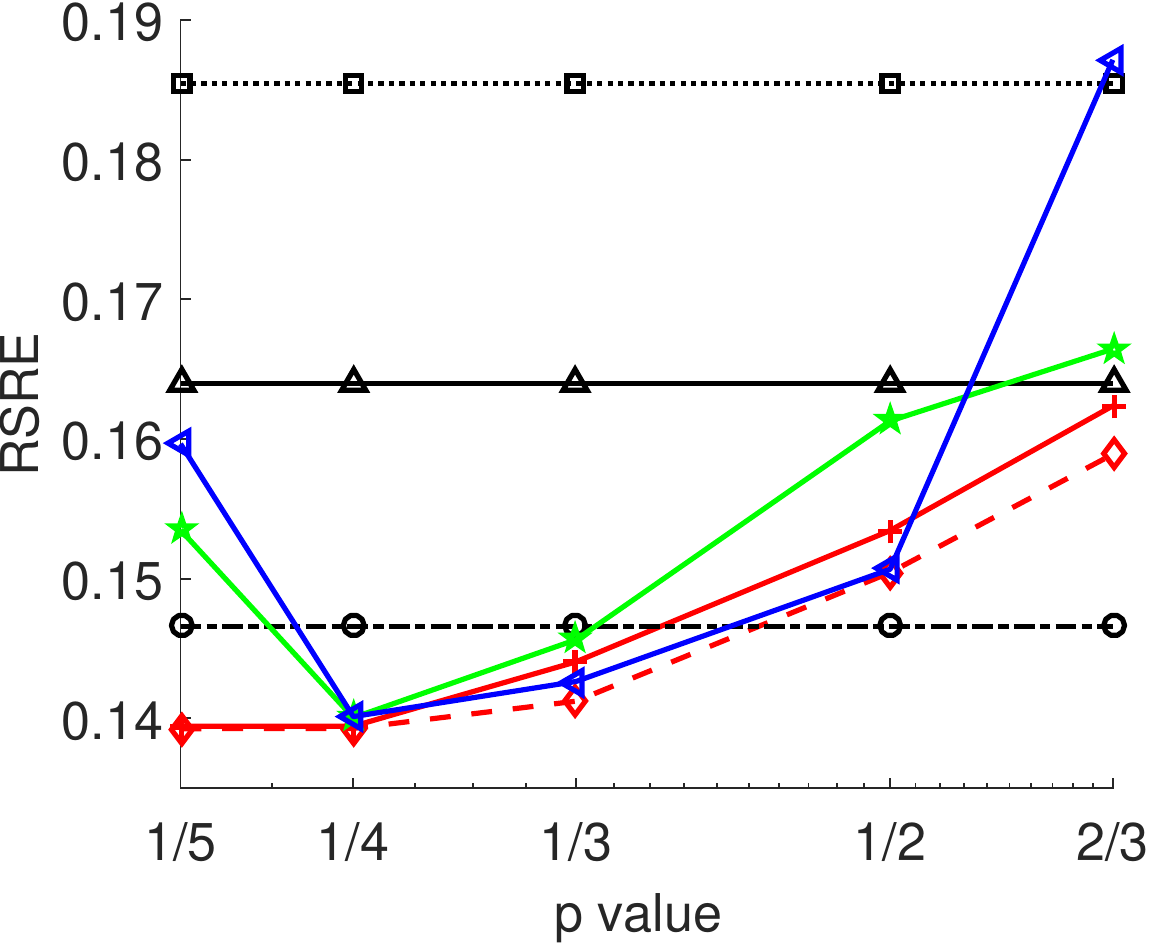}&
				{\hspace{-13pt}} \includegraphics[width=0.24\linewidth,height=0.21\linewidth]{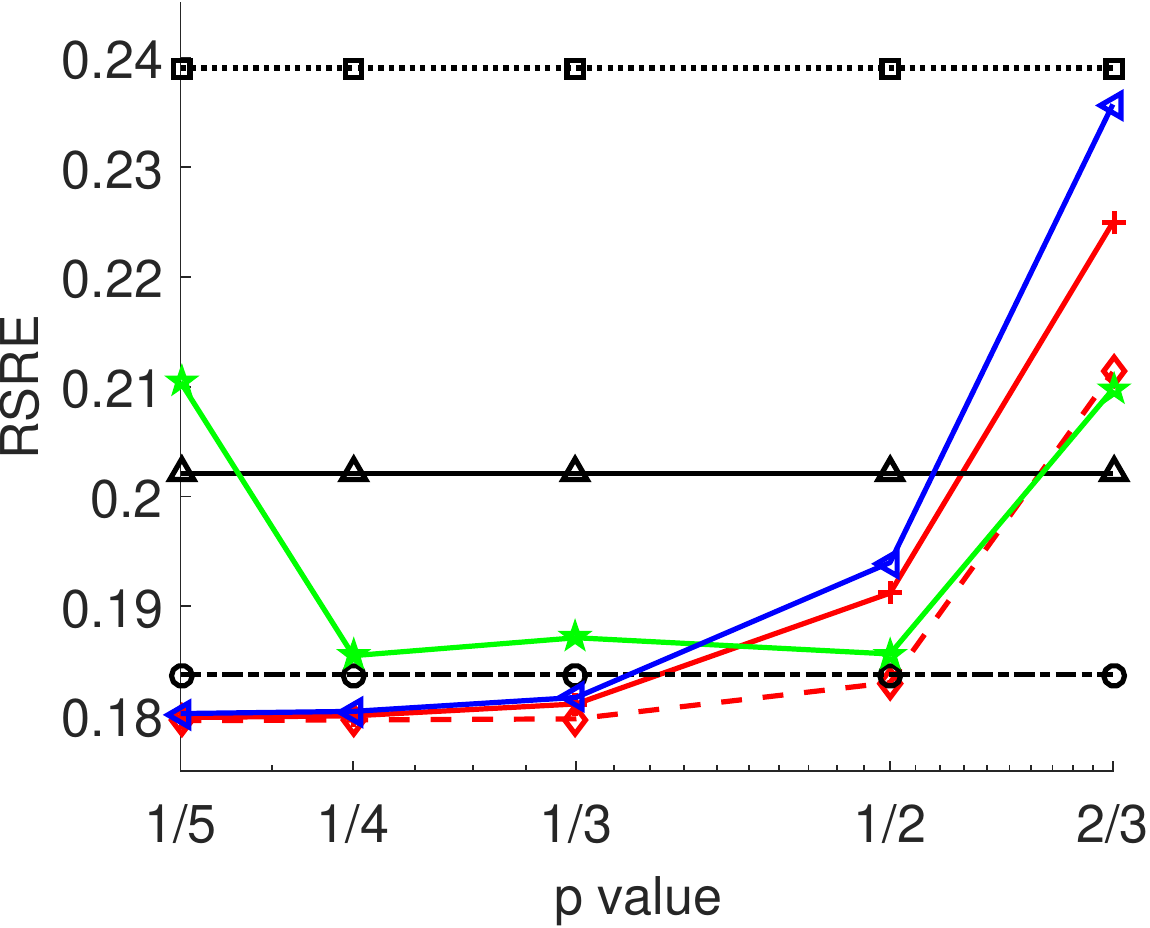}&
				{\hspace{-13pt}} \includegraphics[width=0.24\linewidth,height=0.21\linewidth]{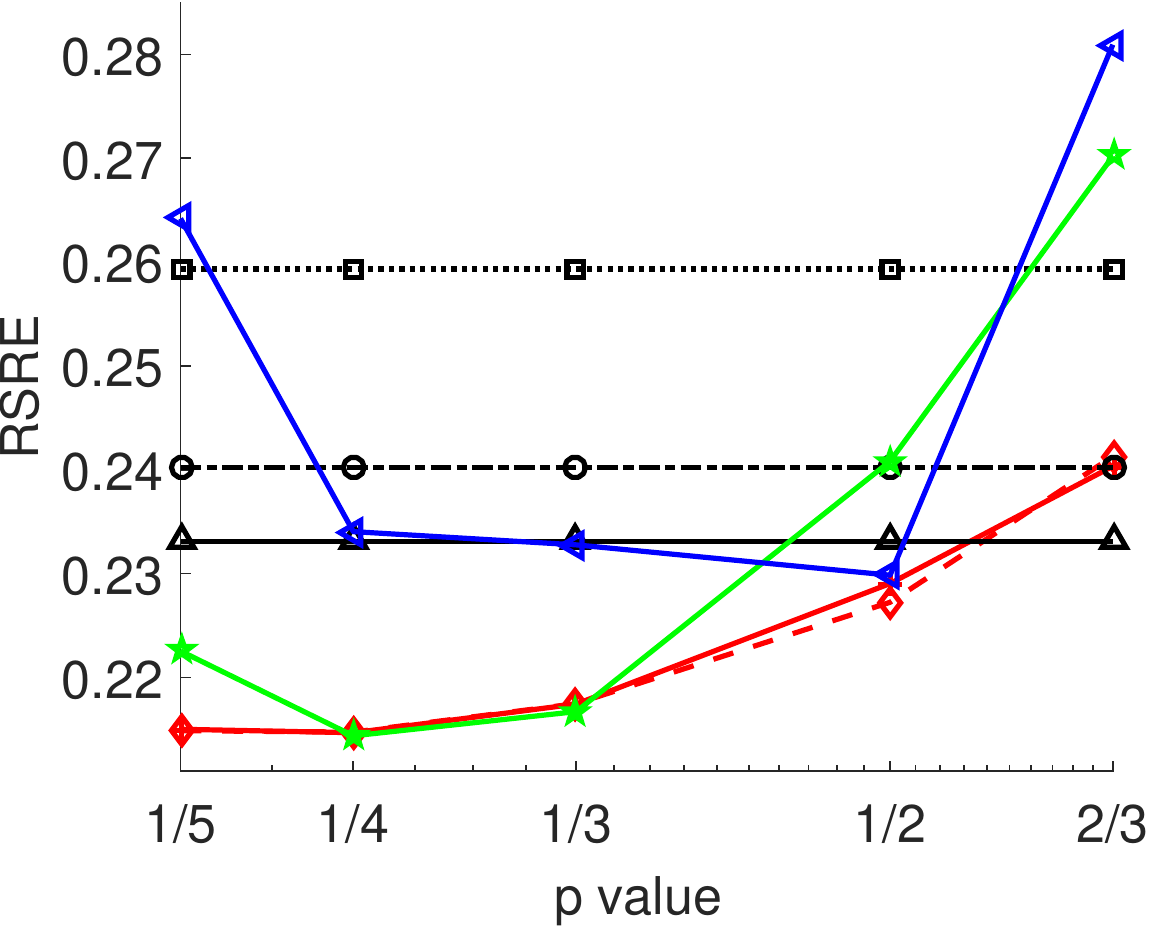}&
				{\hspace{-13pt}} \includegraphics[width=0.24\linewidth,height=0.21\linewidth]{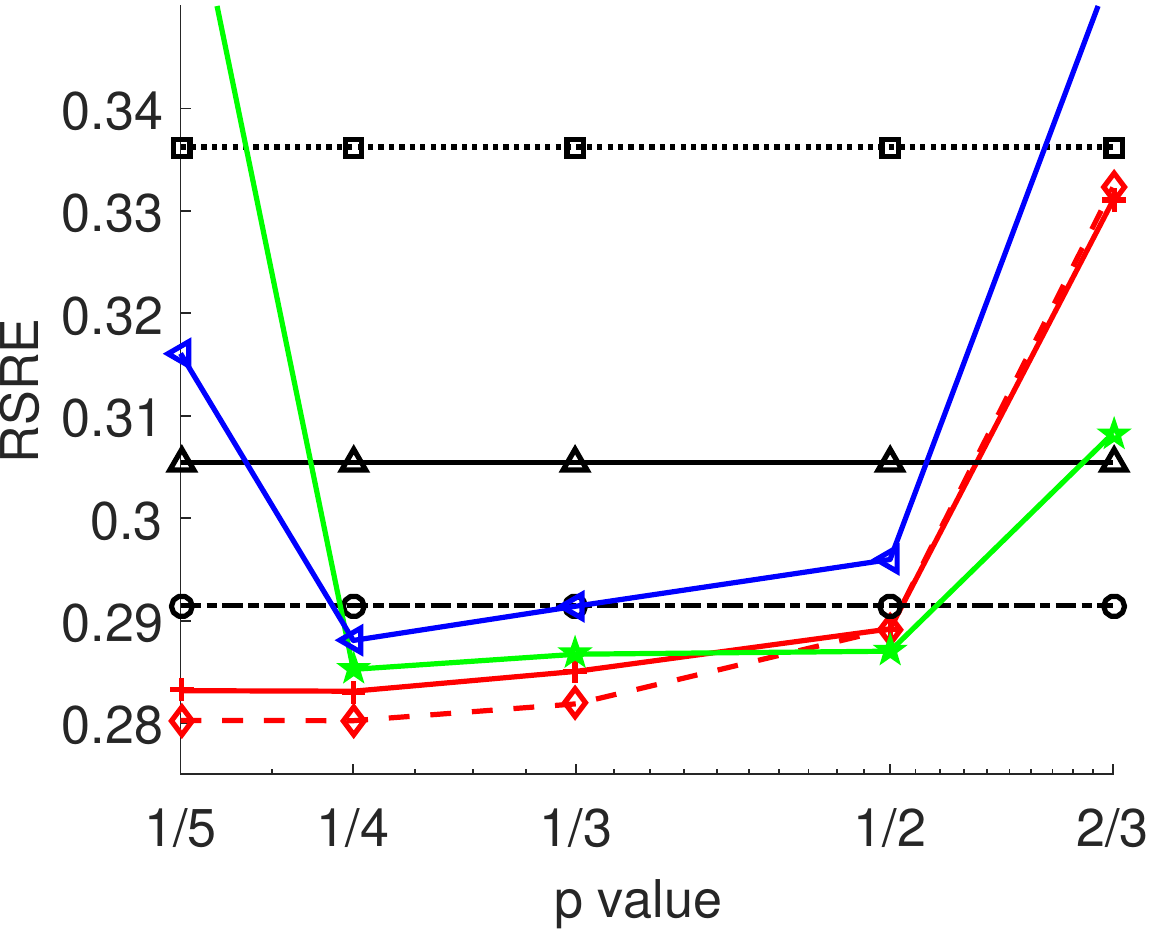}
				{\vspace{-3.5pt}}\\
				\scriptsize{\hspace{-13pt}} (a) $\sigma=0.3, o\%=20\%$ &  \scriptsize{\hspace{-13pt}} (b) $\sigma=0.4, o\%=20\%$ &
				\scriptsize{\hspace{-13pt}} (c) $\sigma=0.3, o\%=15\%$ & \scriptsize{\hspace{-13pt}} (d) $\sigma=0.4, o\%=15\%$ \\
			\end{tabular}\vspace{-2mm}
			\caption{Synthetic experiments on the Schatten-$p$ norm regularized algorithms, i.e.,  our MSS and BSS, F/N \citep{scalableAlgorithm}, BiN \citep{scalableAlgorithm}, TriN \citep{aistats}, IRucLq \citep{iruclq}, and IRNN \citep{irnn}, with varying $p$ values, noise magnitude $\sigma$, and observed data percentage $o\%$. As F/N, BiN, and TriN are only for $p=2/3$, $p=1/2$, and $p=1/3$, respectively, we plot them across different $p$ values for comparison with others.}\vspace{-3mm}
			\label{synthetic}
		\end{figure*}
	
	For better reference, we summarize the algorithm for minimizing  $F(\mathcal{X})$ in \eqref{matrixcompletion} in Algorithm \ref{algorithm} . The running time is dominated by performing matrix multiplications. The total time complexity is $\mathcal{O}(mnd)$, where $d \ll \min \{m,n\}$. We terminate the algorithm when all the magnitudes of gradients \eqref{matrixgradient} over the Lipschitz constants in \eqref{matrixl}  are below a threshold.
	For further acceleration, we adopt the backtracking continuation technique \citep{nnls} to find a proper local Lipschitz constant instead of the global ones as shown in \eqref{matrixl}. Namely, we initially underestimate $L_{i}^{k-1}$ by multiplying a factor $\rho<1$.  We then increase $\rho$ gradually along the iteration until it approaches the upper bound, i.e., $1$. As pointed out in \citep{mmnonconvex}, such a technique can further improve the quality of the solution for non-convex optimization. Note that the inner factors are of much smaller size while taking almost the same updating cost as the side factors, i.e.,  $X_1$ and $X_I$. When the number of factors $I$ becomes big, the redundancy on the inner factors may weaken the effectiveness on decreasing the objective. Thus we modify Algorithm \ref{algorithm} by updating only one or two inner factors in a shuffling order in each cycle (lines $2-6$).  The convergence result remains unchanged, which is as follows:

	\begin{theorem} (Sequence Convergence) 
		Let $\{(X_1^{k},\ldots, X_I^{k})\}$ be a sequence generated by Algorithm \ref{algorithm} with all $p_i > 0$ being rational, then it is a Cauchy sequence and converges to a critical point of \eqref{matrixcompletion}.   
	\end{theorem}
		Note that the sequence convergence is stronger than those in the existing general Schatten-$p$ solvers,
	e.g., IRucLq  \citep{iruclq} and  IRNN \citep{irnn}, where they only prove that any limit point is a stationary point, which is \emph{subsequence} convergent.
	
		
	\section{Experiments} \label{experiments}
	We test our framework in two variants: Bi-Schatten-$p$ norm Surrogate (BSS) with $p_1=p_2=2p$ and Multi-Schatten-$p$ norm Surrogate (MSS) with $p_i \geq 1$ and extrapolation as \eqref{hatx}.  We do not use extrapolation for BSS when the $p_1$ and $p_2 \geq 1$ in order to test  the effectiveness of this technique by comparing with MSS.  We compare them with several state-of-the-art algorithms for Schatten-$p$ norm regularized problems, i.e., F/N ($p=2/3$) \citep{scalableAlgorithm}, BiN ($p=1/2$)  \citep{scalableAlgorithm}, TriN ($p=1/3$) \citep{aistats}, IRNN\footnote{\url{https://sites.google.com/site/canyilu/}} \citep{irnn} and IRucLq \footnote{\url{http://www.math.ucla.edu/~wotaoyin/}} \citep{iruclq}. As IRNN and IRucLq are not suitable for large-scale  problems due to high computing costs and memory requirements, we also include the state-of-the-art solvers for matrix completion, i.e., NNLS\footnote{\url{http://www.math.nus.edu.sg/~mattohkc/NNLS.html}} \citep{nnls}, LMaFit\footnote{\url{http://lmafit.blogs.rice.edu/}} \citep{lmafit2}, and Soft-ALS\footnote{\url{http://cran.r-project.org/web/packages/softImpute/}} \citep{softals} on real datasets. We implement F/N, BiN, and TriN by ourselves, which also use PALM.  Their main differences from ours are that they do not  employ either the continuation technique or the extrapolated-based acceleration. We initialize all algorithms with the same random matrices.
	All the codes are run in Matlab on a desktop PC with a $3.4$ GHz CPU and $20$ GB RAM.
	
	\subsection{Synthetic Data}
	We first generate synthetic data matrices $M=U_0V_0^T$, where $U_0 \in \mathbb{R}^{100\times 5}$ and $V_0 \in \mathbb{R}^{100\times 5}$. The entries of $U_0$ and $V_0$ are sampled i.i.d. from the standard Gaussian distribution $\mathcal{N}(0,1)$. Then $\mathcal{N}(0,\sigma)$ Gaussian noise is added independently to every entry of $M$ and  a portion ($\%$) of them are picked out uniformly as the observed data.  
	We conduct experiments by choosing different $p$ values, i.e., $p= 1/5, 1/4, 1/3, 1/2$, and $2/3$, varying the noise magnitude $\sigma= 0.3,0.4$ and  the observed data percentage $o\%=15\%, 20\%$. At each combination of these hyper parameters, we repeat the experiments $30$ times. All the compared algorithms use the same data in each trial.  
	For MSS with multiple factors, we set all $p_i=1$ when $p=1/5, 1/4$, or  $1/3$ and all $p_i=2$ when $p=1/2$ or $2/3$\footnote{In fact, we have tested different numbers of factors (corresponding to different $p_i$'s) on the same $p$ while find no  difference in the measured performance.}. To cope with different Schatten-$p$ norm regularizers, the regularization parameter  $\lambda$ of all compared algorithms is tuned in the range $[1,20]$. As done in \cite{iruclq}, $d$ is overestimated as $3 \times 5=15$. 
	
	We use the relative square root error (RSRE), i.e., $\|X-U_0V_0^T\|_F/\|U_0V_0^T\|_F$,  to evaluate the performance of recovery.  The average values over all trails are shown in Fig. \ref{synthetic}. Among all the compared algorithms, our MSS achieves the most plausible performance, which reports the least or the second least RSREs in most cases. Besides, MSS shows less sensitivity to the change of $p$ values than other general Schatten-$p$ norm solvers, i.e., IRNN, IRucLq, and BSS.  At $p=2/3$ ($1/2$ or $1/3$), F/N (BiN or TriN) is inferior to our BSS and  MSS, which confirms the effectiveness of the continuation (and extrapolation for MSS) technique.  Across all the four subfigures, all the general Schatten-p norm solvers achieve reasonably good performance at $p=1/4$.  So in the rest of this paper, we only report the results with $p=1/4$.
	\begin{figure*}[tb]
		\centering
		\includegraphics[width=0.6\linewidth]{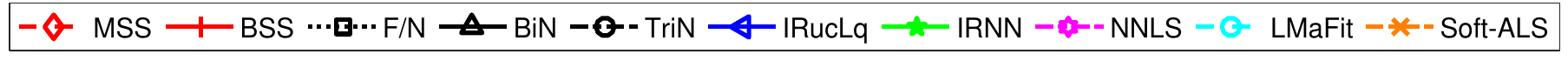} {\vspace{0.07in}}
		\begin{tabular}{ccc}
			{\hspace{-13pt}} \includegraphics[width=0.32\linewidth,height=0.29\linewidth]{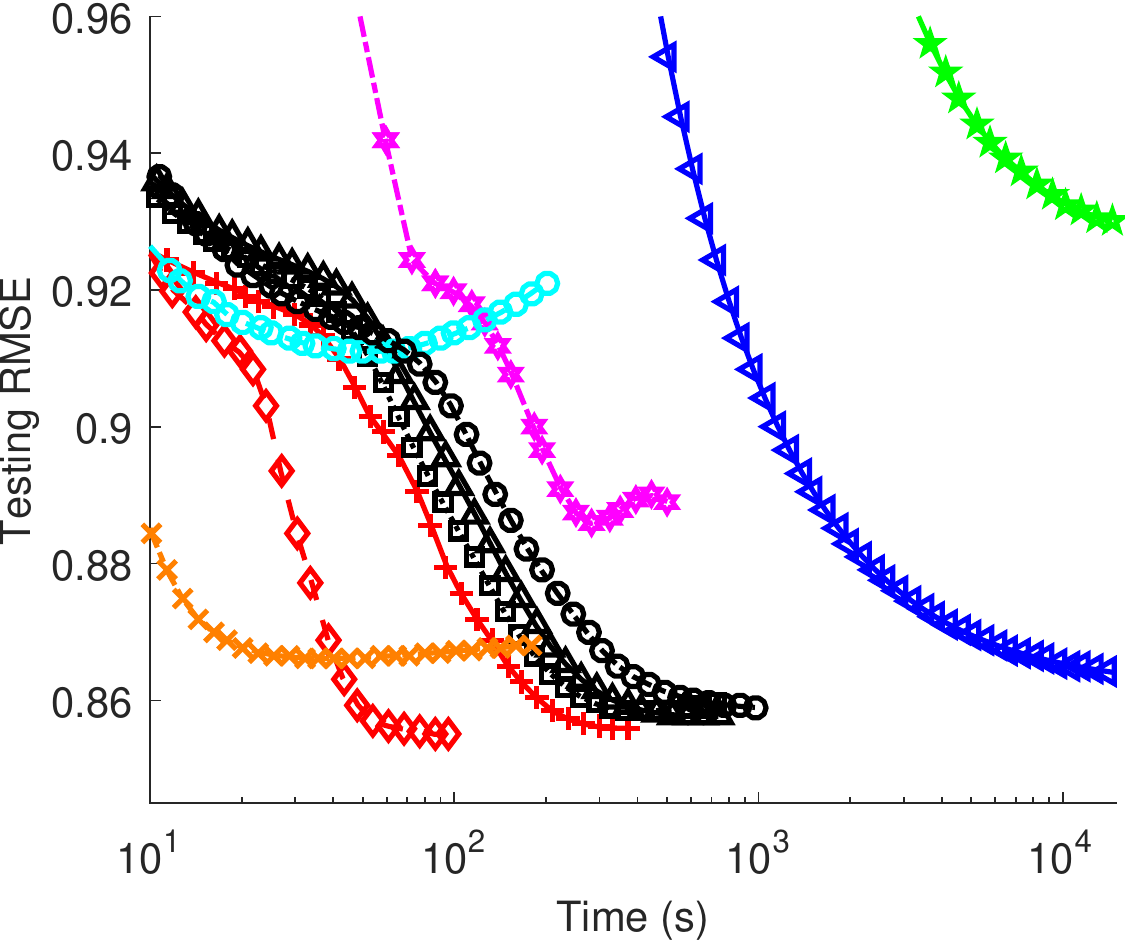}&
			{\hspace{-13pt}} \includegraphics[width=0.32\linewidth,height=0.29\linewidth]{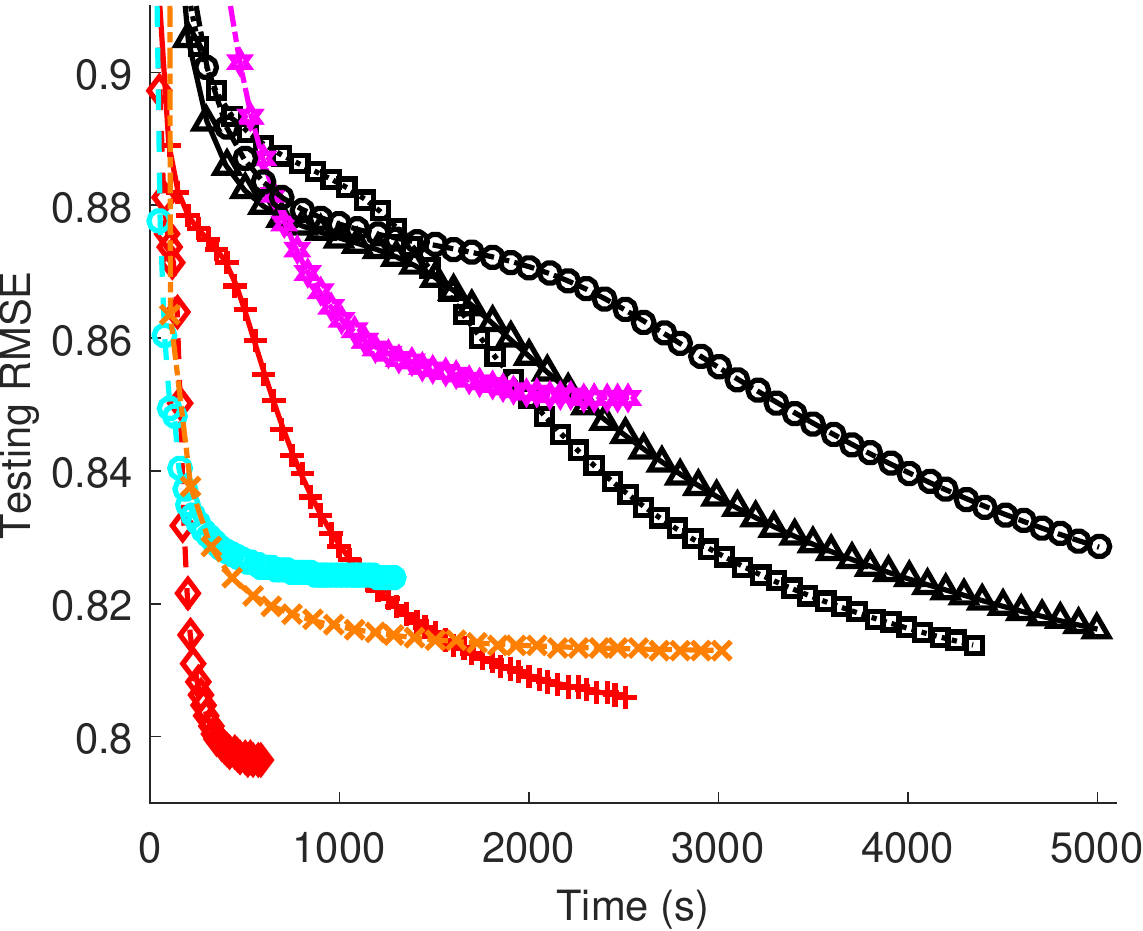}&
			{\hspace{-13pt}} \includegraphics[width=0.32\linewidth,height=0.29\linewidth]{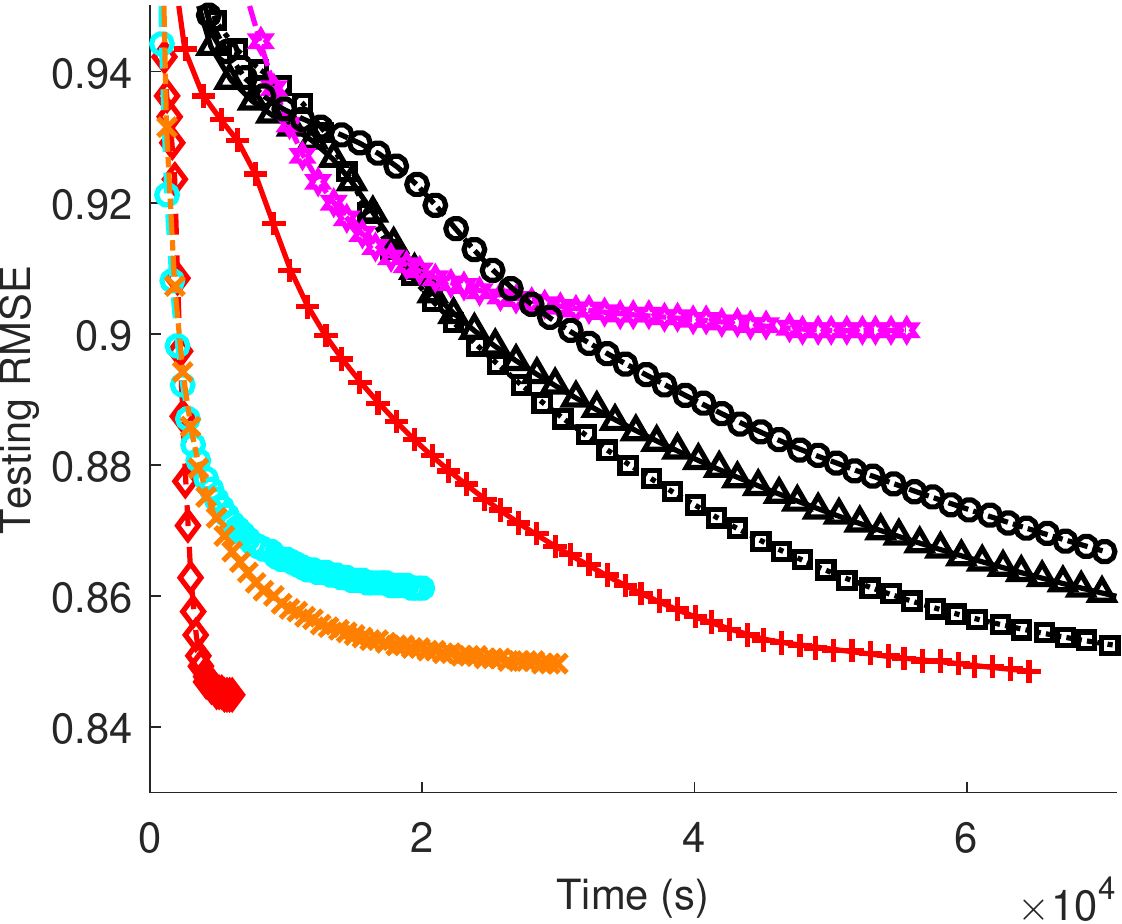}
			{\vspace{-3.5pt}}\\
			
			{\hspace{-13pt}} \includegraphics[width=0.32\linewidth,height=0.29\linewidth]{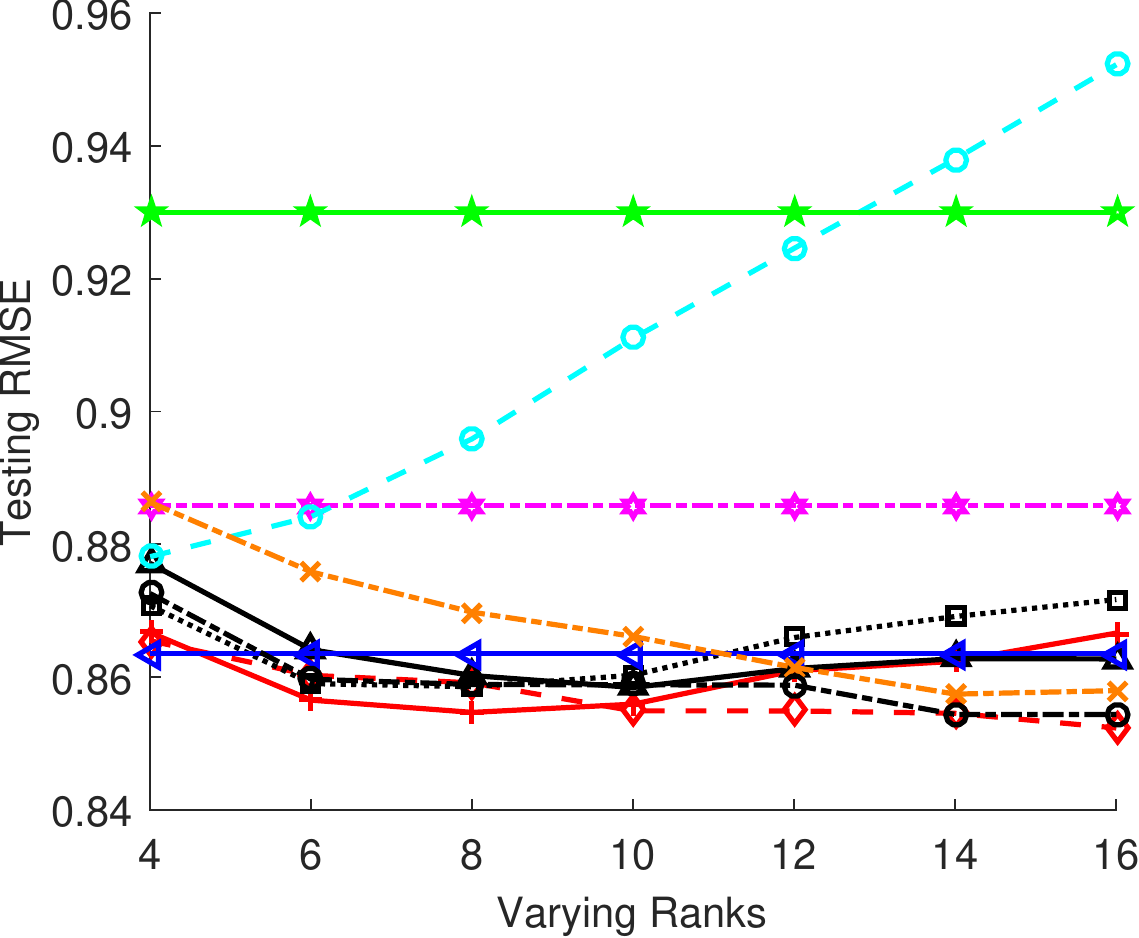}&
			{\hspace{-13pt}} \includegraphics[width=0.32\linewidth,height=0.29\linewidth]{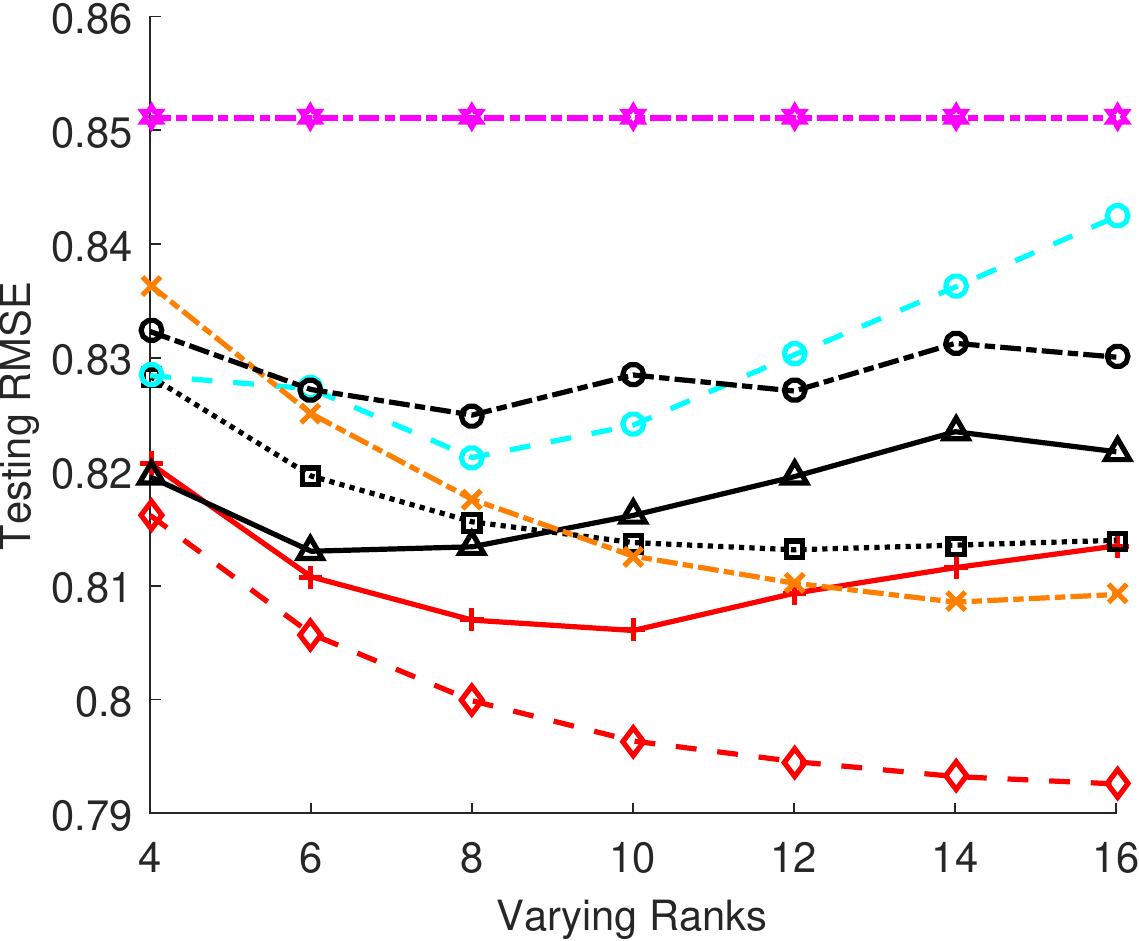}&
			{\hspace{-13pt}} \includegraphics[width=0.32\linewidth,,height=0.29\linewidth]{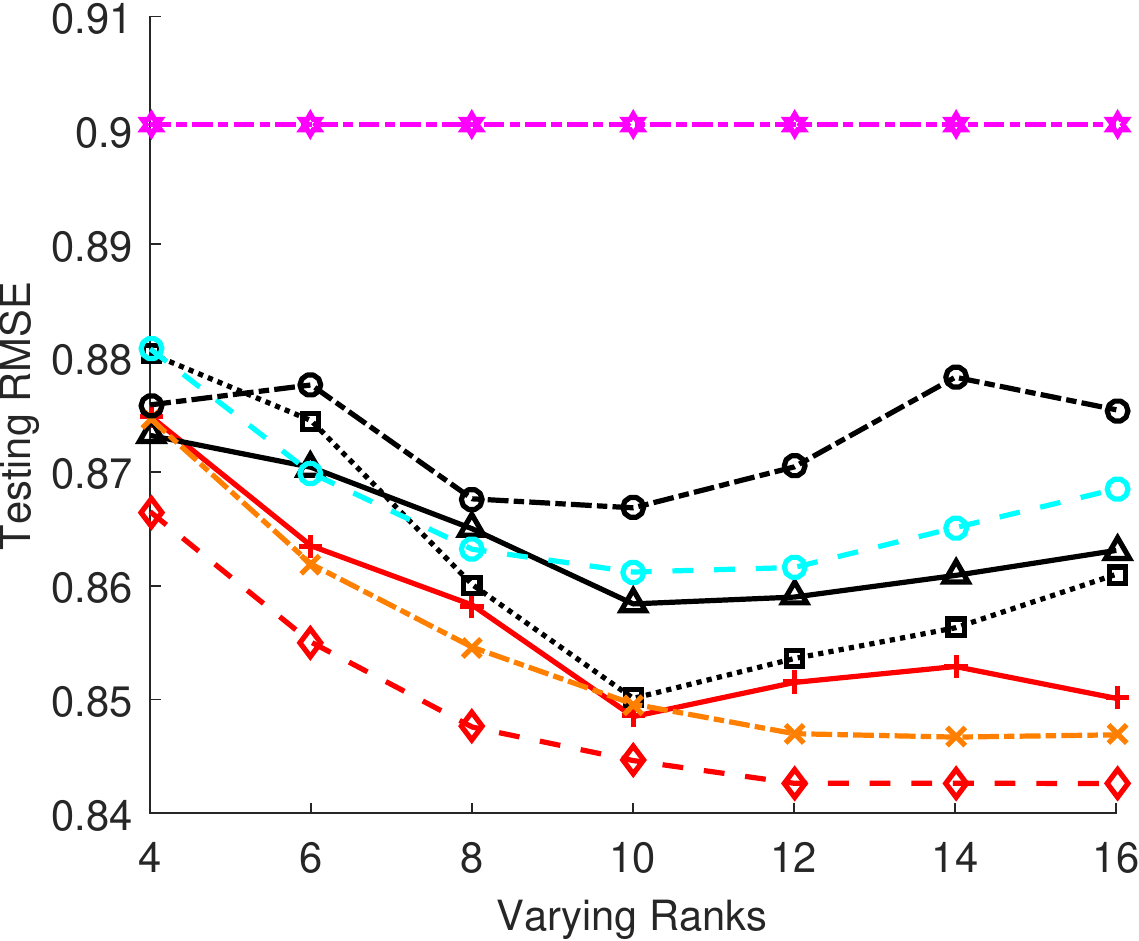}
			{\vspace{-3.5pt}}\\
			\scriptsize{\hspace{-13pt}} (a) MovieLens $1$M &  \scriptsize{\hspace{-13pt}} (b) MovieLens $10$M &
			\scriptsize{\hspace{-13pt}} (c) Netflix\\
		\end{tabular}
		\caption{Matrix completion on recommendation system data sets using  our MSS and BSS, F/N \citep{scalableAlgorithm}, BiN \citep{scalableAlgorithm}, TriN \citep{aistats}, IRucLq \citep{iruclq}, IRNN \citep{irnn}, NNLS \citep{nnls}, LMaFit \citep{lmafit2}, Soft-ALS \citep{softals}. In the first row, we depict the testing RMSE along the executing time ($d=10$ for the factorization formulation). Note that the subfigure (a) is in log-10 scale in the time axis. IRucLq and IRNN are not included in the last two datasets due to their unaffordable computing time and high memory requirements. The second row tests the sensitivity w.r.t. varying ranks $d$.} 
		{\vspace{-0.5em}}
		\label{rs}
	\end{figure*}
	\subsection{Real Data}
	We conduct experiments on three real-world recommendation system datasets: MovieLens $1$M, MovieLens $10$M\footnote{\url{http://www.grouplens.org/node/73}},  and Netflix \citep{kdd}. The corresponding observed matrices are of size $6040 \times 3449$ with $o\%=4.80\%$,  $69878 \times 10677$ with $o\% = 1.34\%$, and  $480189 \times 17770$ with  $o\% = 1.18\%$, respectively. Here we fix the regularization $\lambda=200$ and tune it for other algorithms  in the range $[1,200]$. Following the experimental setup in \citep{scalableAlgorithm}, we randomly pick out  $80\%$ of the observed entries as the training data and use the remaining $20\%$ for testing. 
	The root mean squared errors (RMSEs) on the test set $T$, i.e., $\sqrt{ \sum_{(i,j)\in T}(X_{ij}-M_{ij})^2/ |T|}$,  are measured during the computation. Due to the non-convexity of the Schatten-$p$ norm solvers, we repeat them with different random initializations on MovieLens $1$M and $10$M, while find no significant difference. Results ($d=10$ for the factorization formulation) of all compared  algorithms  are shown in the first row of Fig. \ref{rs}. As IRNN and IRucLq are slow and require large memory,  we do not apply them to MovieLen $10$M and Netflix. From the figure,  we can see some obvious gaps in the time vs. testing RMSE curves between our BSS and MSS. This is mainly caused by the extrapolation technique employed by MSS. The testing RMSE of LMaFit increases after several iterations on MovieLens $1$M, which may be caused by the intrinsic unregularized model. Among all compared algorithms,  our MSS achieves the best performance across all the three data sets. It reaches the smallest testing RMSEs with the least time in less than $500$ iterations. We also conduct experiments to test the sensitivity with respect to different estimated ranks $d$. The results are shown in the second row, where our  MSS shows an apparent superiority over others, especially on large-scale datasets, i.e., MovienLen $10$M  and Netflix.

	\section{Conclusions}
	In this paper, we propose a unified surrogate for the Schatten-$p$ norm with two factor matrix norms.  We further extend it to multiple factor matrices and show that  all the factor norms can be \textit{convex and smooth} for any $p>0$. In contrast, the original Schatten-$p$ norm for $0<p<1$ is non-convex and non-smooth. We establish equivalence between the surrogate formulation and the original problem and show that the former should be preferred in practice. As an example we conduct experiments on matrix completion. By utilizing the convexity of the factor matrix norms, our accelerated PALM achieves the state-of-the-art performance.  Its sequence convergence is also established. 
	\section{Acknowledgements}
	Zhouchen Lin is supported by National Basic Research Program of China (973 Program) (grant no. 2015CB352502), National Natural Science Foundation (NSF) of China (grant nos. 61625301 and 61231002), and Qualcomm.
	\small{
		\bibliographystyle{aaai}
		\bibliography{refs}}
	\newpage
	
	\begin{center}
		\LARGE\textbf{Supplementary Material}
	\end{center}
	\section{Proofs}
	\subsection{Proof of Theorem 1}
	
	\begin{proof}
		As $X=UV^T$, where $U \in \mathbb{R}^{m \times d}$, $V \in \mathbb{R}^{n \times d}$, for any $p, \mu, \nu > 0$ with $1/\mu+1/\nu =1$, we have
		\begin{equation} \label{relaxqp}
			\begin{split}
				&\sum_{i=1}^{\min\{m,n,d\}} \sigma_i^p(X) \\\leq & \sum_{i=1}^{\min\{m,n,d\}} \sigma_i^p(U) \sigma_i^p(V)  
				\\\leq &  \left(\sum_{i=1}^{\min\{m,n,d\}}\sigma_i^{p\mu}(U)\right)^{\frac{1}{\mu}}\left(\sum_{i=1}^{\min\{m,n,d\}}\sigma_i^{p\nu}(V)\right)^{\frac{1}{\nu}} \\
				\leq & \frac{1}{\mu} \left(\sum_{i=1}^{\min\{m,n,d\}}\sigma_i^{p\mu}(U)\right)+\frac{1}{\nu}\left(\sum_{i=1}^{\min\{m,n,d\}}\sigma_i^{p\nu}(V)\right) \\
				\leq & \frac{1}{\mu} \left(\sum_{i=1}^{\min\{m,d\}}\sigma_i^{p\mu}(U)\right)+\frac{1}{\nu}\left(\sum_{i=1}^{\min\{n,d\}}\sigma_i^{p\nu}(V)\right) 
			\end{split}
		\end{equation}      
		where the first inequality follows from Lemma 2 in the paper. The second inequality holds due to the Holder’s inequality, i.e., $\sum_{k=1}^n |x_k y_k| \leq (\sum^n_{k=1} |x_k|^{\mu})^{1/\mu}(\sum^n_{k=1} |y_k|^{\nu})^{1/\nu}$ with $1/\mu + 1/\nu = 1$.  The third inequality holds due to the Jensen’s inequality for the concave function $f(x) = log(x)$ by taking logarithm on both sides. The forth inequality is derived from the fact that $\min\{m,n,d\} \leq  \min\{m,d\}$ and $\min\{m,n,d\} \leq  \min\{n,d\}$.
		
		As ${1}/{p}=1/p_1+1/p_2$, substituting $\mu=p_1/p$ and $\nu= p_2/p$ in \eqref{relaxqp}, we have 
		\begin{equation}
			\begin{split}
				\frac{1}{p} \sum_{i=1}^{\min\{m,n,d\} } \sigma_i^p(X) \leq & \frac{1}{p_1} \left(\sum_{i=1}^{\min\{m,d\}}\sigma_i^{p_1}(U)\right)\\
				+&\frac{1}{p_2}\left(\sum_{i=1}^{\min\{n,d\}}\sigma_i^{p_2}(V)\right)\end{split}
		\end{equation}
		By $d \geq r$ and the definition of the Schatten-$p$ norm, we have
		\begin{equation} \label{xuv}
			\frac{1}{p} \|X\|_{S_p}^p \leq \frac{1}{p_1}\|U\|_{S_{p_1}}^{p_1}+ \frac{1}{p_2} \|V\|_{S_{p_2}}^{p_2}.
		\end{equation}
		
		When $d \leq \min\{m,n\}$, denote  $X = U_X \Sigma_X V_X^T$ as the SVD of $X$, where
		$U_X \in \mathbb{R}^{m \times d}$, $V_X \in  \mathbb{R}^{n \times d}$, and $\Sigma_X = \mathrm{diag}([\sigma_1(X),\ldots, \sigma_r(X), 0,\ldots, 0]) \in \mathbb{R}^{d\times d}$. Let $U^*=U_X \Sigma_X^{p/p_1}$ and $V^* =V_X\Sigma_X^{p/p_2}$, where $\Sigma^x$ is entry-wise power to $x$, then we have $X = U^* V^{*T}$ and 
		\begin{equation}
			\frac{1}{p} \|X\|_{S_p}^p = \frac{1}{p_1}\|U^*\|_{S_{p_1}}^{p_1}+ \frac{1}{p_2} \|V^*\|_{S_{p_2}}^{p_2}.
		\end{equation}
		
		When $d > \min\{m,n\}$, the above equality also holds by adding extra $m-d$ (or $n-d$) zero columns to $U_X$ (or $V_X$).
		
		This completes the proof.
	\end{proof}
	\subsection{Proof of Corollary 2}
	\begin{proof}
		Without loss of generality, we assume that $d_i =d  \leq \min\{m,n\}, i =1, \ldots, I$. As the inequality of \eqref{xuv} naturally extends to $I \geq 2$ factor matrices with $X=\prod_{i=1}^{I}X_i$ and ${1}/{p}= \sum_{1}^I 1/p_i$, we only need to show that the equality can hold. Let $X_1^* = U_X \Sigma_X^{p/p_1}$, $X_i^* = \Sigma_X^{p/p_i}, i=2,\ldots,I-1$, and $X_I^* = \Sigma_X^{p/p_I} V_X^T$, then we have $X=\prod_{i=1}^{I}X_i^*$ and
		\begin{equation} \label{general2}
			\begin{split}
				\frac{1}{p} \|X\|_{S_p}^p
				= \min_{X_i:X=\prod_{i=1}^{I}X_i} \sum_{i=1}^{I} ~\frac{1}{p_i}\|X_i^*\|_{S_{p_i}}^{p_i}.
			\end{split} 
		\end{equation} 
		This completes the proof.
	\end{proof}
	
	\subsection{Proof of Proposition 3}
	\begin{proof}
		
		Denote $\lfloor 1/p\rfloor$ as the largest integer not exceeding $1/p$. When $\lfloor 1/p\rfloor=1/p$, we can choose $I=\lfloor 1/p\rfloor$ with all $p_i= 1$. When $\lfloor 1/p\rfloor<1/p$, we can choose $p_i =1, i= 1,\ldots, I-1$ and $p_I=1/(1/p-\lfloor 1/p\rfloor)$. This concludes the proof of the case (a). For the case of  (b), we  choose $I=\lfloor 1/p\rfloor+1$ with $p_i =I p >1$.  
	\end{proof}
	
	\subsection{Proof of Theorem 3}
	\begin{proof}
		As $d \geq r^*$,  any matrix $X$ with rank $r^*$ can be written in the form of $X=\prod_{i=1}^{I}X_i$. The formulation (15) in the paper can be written as
		\begin{eqnarray}
			\label{temp1}
			&  &\min_{X_i,~i = 1, \ldots, I} f \left(\prod_{i=1}^{I}X_i\right) + \sum_{i=1}^{I} ~\frac{\lambda}{p_i}\|X_i\|_{S_{p_i}}^{p_i}\\
			\notag
			&= & \min_{X_i,~i = 1, \ldots, I} f\left(\prod_{i=1}^{I}X_i\right) +~\frac{\lambda}{p}\left\|\prod_{i=1}^{I}X_i\right\|_{S_{p}}^{p} \\
			\notag
			&= &   \min_{X,\mathrm{rank}(X)=r^{*}}  f(X)+~\frac{\lambda}{p}\|X\|_{S_{p}}^{p} \\
			\label{temp2}
			&= &  \quad  \min_{X}  f(X)+~\frac{\lambda}{p}\|X\|_{S_{p}}^{p},
		\end{eqnarray}	
		where the first equation is by Corollary 2 and the last equality is by the fact that the rank of the solution to \eqref{temp2} is $r^*$.  The equivalence of the criteria in \eqref{temp1} and \eqref{temp2} completes the proof.
	\end{proof}
	
	\subsection{Proof of Theorem 4}
	Before giving our proof, we introduce some backgrounds.
	\begin{definition} (Semi-algebraic sets and functions \citep{palm}).  A subset $S \subset \mathbb{R}^n$  is a real semi-algebraic set if there exists a finite number of real polynomial functions $g_{ij}, h_{ij}: \mathbb{R}^{n} \rightarrow \mathbb{R}$  such that
		\begin{equation*}
			S=\bigcup_{j}^{n_j} \bigcap_{i}^{n_i} \{ u  \in \mathbb{R}^{n}: g_{ij}(u)=0~\text{and}~ h_{ij}(u) <0\}.  
		\end{equation*}
		Moreover, a function $f: \mathbb{R}^n \rightarrow (-\infty, +\infty] $ is called semi-algebraic if its graph $\{(u,t)\} \in \mathbb{R}^{n+1} : f(u) = t\}$ is a semi-algebraic set.
	\end{definition}
	Semi-algebraic sets are stable under the operations of finite union, finite intersections, complementation and Cartesian product. The following are the semi-algebraic functions or the property of semi-algebraic functions used below:
	\begin{itemize}
		\item Real polynomial functions.
		\item Finite sums and product of semi-algebraic functions.
		\item Composition of semi-algebraic functions.
	\end{itemize}
	\begin{proposition} \label{schattensemi}
		The matrix operator $\|\cdot\|_{S_p}^p$ with $p >0$ being rational is a semi-algebraic function. 
	\end{proposition}
	\begin{proof}
		Consider a matrix $X \in \mathbb{R}^{m \times n}$. We can rewrite $\|X\|_{S_p}^p=\sum_{i=1}^{m} e_i^{p/2}(XX^T)$, where $e_i$ represents the $i$-th eigenvalue of $XX^T$. As the eigenvalues are roots of the corresponding characteristic polynomial function, thus being semi-algebraic. Combining with that the operator $\| \cdot\|_p^p$ on vectors is semi-algebraic with $ p>0$ being rational (\citep{palm} [Example 4]), it is natural to conclude that their composition, i.e., the matrix operator $\|\cdot\|_{S_p}^p$, is also semi-algebraic.

	\end{proof}
	
	As Algorithm 1 is a special case of the general optimization algorithm in \citep{nonconvexabcd}, to prove the convergence of Algorithm 1, we only need to show that the conditions ensuring  sequence convergence of \citep{nonconvexabcd} [Theorem 2.7] are satisfied.
	\begin{proposition}\label{algorithmconditions} The sequence generated by Algorithm 1 is a Cauchy sequence and converges to a critical point if the following conditions hold:
		\begin{enumerate}[(a)]
			\item The sequence of $F(\mathcal{X}^k)$  is non-increasing. \label{noninc}
			\item Within any bounded consecutive iterations, every block $X_i$ is updated at least one time. \label{onetime}
			\item $\{\mathcal{X}^k\}$ is a bounded sequence. \label{bounded}
			\item $F$ is a semi-algebraic function. \label{semi}
			\item  $\nabla_{{i}}^k f(X_{i})$ has a Lipschitz constant $L^{k-1}_i$ with respect to $X_{i}$,
			and there exist constants $0<\ell\le L<\infty,$ such that $\ell\le L_i^{k-1}\le L$ for all $k$ and $i$. \label{lp1}
			\item $\nabla_{\mathcal{X}} f(\mathcal{X})$ has Lipschitz constant on any bounded set. \label{lp2} 
		\end{enumerate}
	\end{proposition}
	
	Now we are ready to prove Theorem 4 by verifying that Algorithm 1 satisfies all the conditions in Proposition \ref{algorithmconditions}.
	\begin{proof}
		Conditions \eqref{noninc} and \eqref{onetime} naturally hold by the update of Algorithm 1.
		By the non-increasing property, we have
		\begin{equation}
			F(\mathcal{X}_1) \geq F(\mathcal{X}_k) \geq \sum_{i=1}^{I} ~\frac{\lambda}{p_i}\|X_i^k\|_{S_{p_i}}^{p_i}
		\end{equation} 	    
		Then Condition \eqref{bounded} holds.
		
		As  $f(\mathcal{X})=\frac{1}{2} \left\|W \odot (M - \prod_{i=1}^{I}X_i) \right\|_F^2$, which is polynomial, it is a semi-algebraic function. Combining with Proposition \ref{schattensemi}, Condition \eqref{semi} holds.
		
		We retell the Equation (20) in the paper as follows:
		\begin{equation}
			L_i^{k-1}=\max\{\|A_{-i}^{k}\|_2^2 \|A_{+i}^{k-1}\|_2^2 , \epsilon \},
		\end{equation} 
		where $A_{-i}^k=X^{k}_1 \cdots X^{k}_{i-1}$ and $A_{+i}^{k-1}=X^{k-1}_{i+1}\cdots X^{k-1}_{I}$. Combining with that $\{\mathcal{X}^k\}$ is bounded, we can always  find an upper bound $L <\infty$ and set the lower bound $\ell = \epsilon$. Thus Condition \eqref{lp1} holds.
		
		Condition \eqref{lp2} naturally holds by following the similar proof as Condition \eqref{lp1} with an arbitrary bounded $\mathcal{X}$. 
	\end{proof}
	
\end{document}